\pdfoutput=1

\documentclass[11pt]{article}
\usepackage{amsfonts} 

\usepackage[preprint]{acl}

\usepackage{times}
\usepackage{latexsym}
\usepackage{amsmath, amssymb, amsthm}

\theoremstyle{plain}
\newtheorem{definition}{Definition}[section]
\newtheorem{lemma}[definition]{Lemma}
\newtheorem{theorem}[definition]{Theorem}

\usepackage[T1]{fontenc}

\usepackage[utf8]{inputenc}

\usepackage{microtype}

\usepackage{inconsolata}

\usepackage{graphicx}
\usepackage{booktabs}   
\usepackage{multirow}   
\usepackage{stfloats}   
\usepackage{amsmath} 
\usepackage{colortbl}
\usepackage{array}
\usepackage{tabularx} 
\usepackage{float}
\usepackage{placeins}  

\title{Mitigating Posterior Salience Attenuation in Long-Context LLMs with Positional Contrastive Decoding}

\author{%
\textbf{Zikai Xiao}\textsuperscript{1,5\thanks{Co-first author.}} 
\textbf{Ziyang Wang}\textsuperscript{2\footnotemark[1]}  
\textbf{Wen Ma}\textsuperscript{1}  
\textbf{Yan Zhang}\textsuperscript{3} \\
\textbf{Wei Shen}\textsuperscript{1} 
\textbf{Yan Wang}\textsuperscript{1} 
\textbf{Luqi Gong}\textsuperscript{4\thanks{Co-corresponding author.}} 
\textbf{Zuozhu Liu}\textsuperscript{1,5\footnotemark[2]} \\
\textsuperscript{1}Zhejiang University, \quad
\textsuperscript{2}University of Science and Technology of China, \quad
\textsuperscript{3}Bytedance, \\
\textsuperscript{4}Zhejiang Lab, \quad
\textsuperscript{5}Zhejiang Key Laboratory of Medical Imaging Artificial Intelligence \\
\texttt{\href{mailto:zikai@zju.edu.cn}{zikai@zju.edu.cn}}
}

\begin{document}
\maketitle
\begin{abstract}
While Large Language Models (LLMs) support long contexts, they struggle with performance degradation within the context window. Current solutions incur high training costs, leaving statistical behaviors and cost-effective approaches underexplored. From the decoding perspective, we identify the Posterior Salience Attenuation (PSA) phenomenon, where the salience ratio correlates with long-text performance degradation. Notably, despite the attenuation, gold tokens still occupy high-ranking positions in the decoding space. Motivated by it, we propose the training-free Positional Contrastive Decoding (PCD) that contrasts the logits derived from long-aware attention with those from designed local-aware attention, enabling the model to focus on the gains introduced by large-scale short-to-long training. Through the analysis of long-term decay simulation, we demonstrate that PCD mitigates attention score degradation. Experimental results show that PCD achieves state-of-the-art performance on long-context benchmarks.

\end{abstract}

\section{Introduction}
The maximum context lengths of large language models (LLMs) have steadily increased, yet their effective utilization remains limited:  most open-source models experience sharp performance degradation beyond $16$k tokens \citep{hsieh2024ruler, hengle2024multilingual, zhang2024bench}. Previous works have sought to explain and improve context utilization. 
From empirical observations derived from the generated text, the "lost in the middle" effect \citet{liu2024lost, zhang2024bench} reveals inconsistent performance drops across different positions \citep{zhang2024bench}, and the "Know but Don’t Tell" phenomenon reveals that while models encode target information, they fail to leverage it in generating accurate responses \citep{lu2024insights}, a gap also examined in theoretical models of reasoning \citep{bi2025cot}.

To fully utilize the context, data-driven methods have been proposed, such as synthetic key-value retrieval mechanisms \citep{an2024make} and multi-document question-answering frameworks \citep{dataartificial}.
In terms of model design, \citet{tworkowski2024focused} enhances attention layers with external memory, while \citet{zhang2024found} leverages Multi-scale Positional Encoding for capturing multi-scale distance awareness. These approaches often involve costly annotation and training.
For inference-time methods, Segment Reranking \citep{dsouza2024evaluating, peysakhovich2023attention} addresses the "lost in the middle" problem at the prompt level by rearranging key segments to the beginning and end of the context. Additionally, \citep{hsieh2024found} estimates and calibrates positional bias using an auxiliary priori dummy document. Prompt-dependent solutions remain brittle and highly sensitive to the hyper-specific prompt formulation. Thus, devising quantitative analysis and cost-effective solutions for long context utilization remains challenging.

In this work, we uncover the Posterior Salience Attenuation (PSA) phenomenon and corresponding Positional Contrastive Decoding (PCD) to make LLMs effectively utilize the context. Through analysis in decoding-space (Section~\ref{sec:psa}), we found that the posterior salience of the gold label gradually degrades as context length grows, when controlling tasks to maintain consistent difficulty. 
Despite the posterior salience decreases, the ranking of the gold label often remains among the top ranks (top 0.006\% in Fig. ~\ref{fig:analysis_exp} (a), regarding the entire vocabulary \cite{dubey2024llama}), suggesting a decoding strategy that amplifies the salience of gold token. Through the analysis of numerous error cases, we found that the model tends to adopt tokens that are closer to the query (proximal tokens), leading to incorrect responses. 
To enhance the distance awareness of LLMs, we proposed PCD that contrasts the logits derived from standard attention scores with those from designed local-aware attention, enabling the model to focus on the gains introduced by large-scale short-to-long training. 
We progressively induce excessive rotation beyond the default angles from high to low frequencies to construct the local-aware attention, inspired by insights from frequency analysis in RoPE \citep{su2024roformer}. 
The numerical and theoretical analysis and experimental analysis demonstrate that PCD effectively enhances long-awareness.

\section{Methods}
In long-context scenarios, denote $f(\theta)$ as LLM with parameters $\theta$. Given a sequence of tokens $\mathbf{x} = [x_1, ..., x_L]$ from the input context, the model $f(\theta)$ generates the next token $y^*$ by computing the posterior probability distribution $P_{f(\theta)}(y^* | \mathbf{x}_{\leq L})$, where $\mathbf{x}_{\leq L}$ represents the input sequence. 
\begin{figure}[h]
    \centering
    \scriptsize
    \includegraphics[width=0.95\columnwidth]{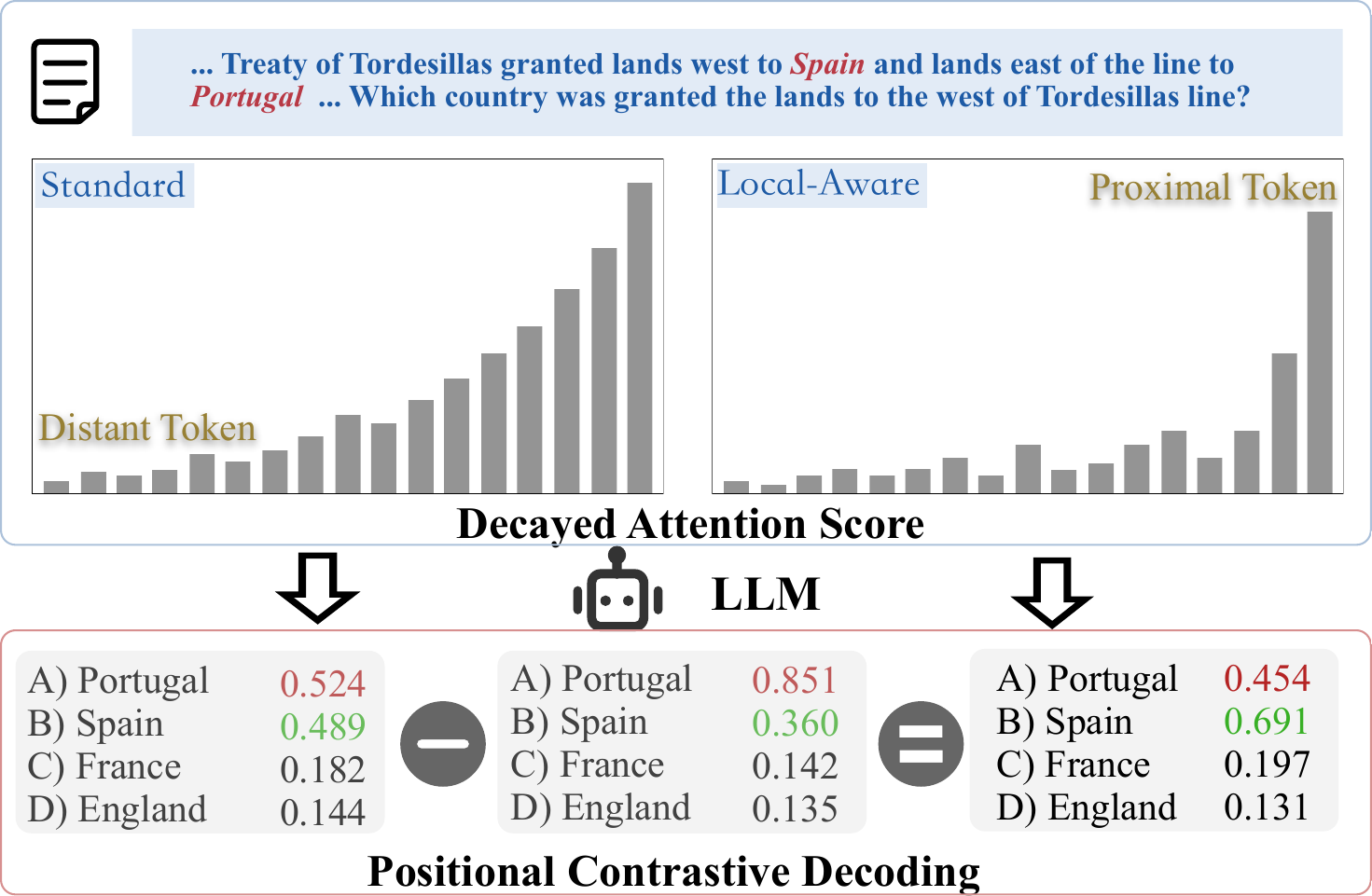} 
    \caption{An illustration of PCD, contrasting logits from long-aware and local-aware attention, and amplifying the gains of large-scale short-to-long training.}
    \vspace{-5pt} 
    \label{fig:pcd_main}
\end{figure}
\subsection{Posterior Salience Attenuation in Long Context LLMs}
\label{sec:psa}
PSA describes a statistical rank decrease of the gold token \(y^*\) in the model’s decoding space. 
 as the context length \(L\) increases.
Concretely, for an input \(x_{\le L}\), Let $ Q $ denote the query set, $ \mathcal{V} $ the vocabulary, and $ y_i^* $ the gold token for the $ i $-th query. we define the \emph{salience score}:
\begin{equation}
\resizebox{\columnwidth}{!}{$
S(L) = \frac{1}{|Q|}\sum_{i=1}^{|Q|}
  \frac{1}{
    1 \;+\; \sum_{v \in \mathcal{V}}
    \mathbb{I}\bigl(
      P_{f(\theta)}(v \mid x_{\le L}^{(i)})
      > 
      P_{f(\theta)}(y_i^* \mid x_{\le L}^{(i)})
    \bigr)
  }
$}
\end{equation}
where $ \mathbb{I}(\cdot) $ is the indicator function, which evaluates to $1$ if the condition inside is true and $0$ otherwise. The $ S(L) $ quantifies the extent to which the model prioritizes $ y_i^* $ by counting the number of tokens $ v \in \mathcal{V} $ whose predicted probabilities exceed that of $ y_i^* $. It then computes the reciprocal of this count to reflect the rank of $ y_i^* $ and averages these reciprocal ranks across all queries in $ Q $. 
A decreasing trend in $ S(L) $ with increasing $ L $ indicates a weakening of the model's ability to prioritize the gold token as the context length grows.

\begin{figure}[t]
    \centering
    \scriptsize
    \includegraphics[width=0.90\columnwidth]{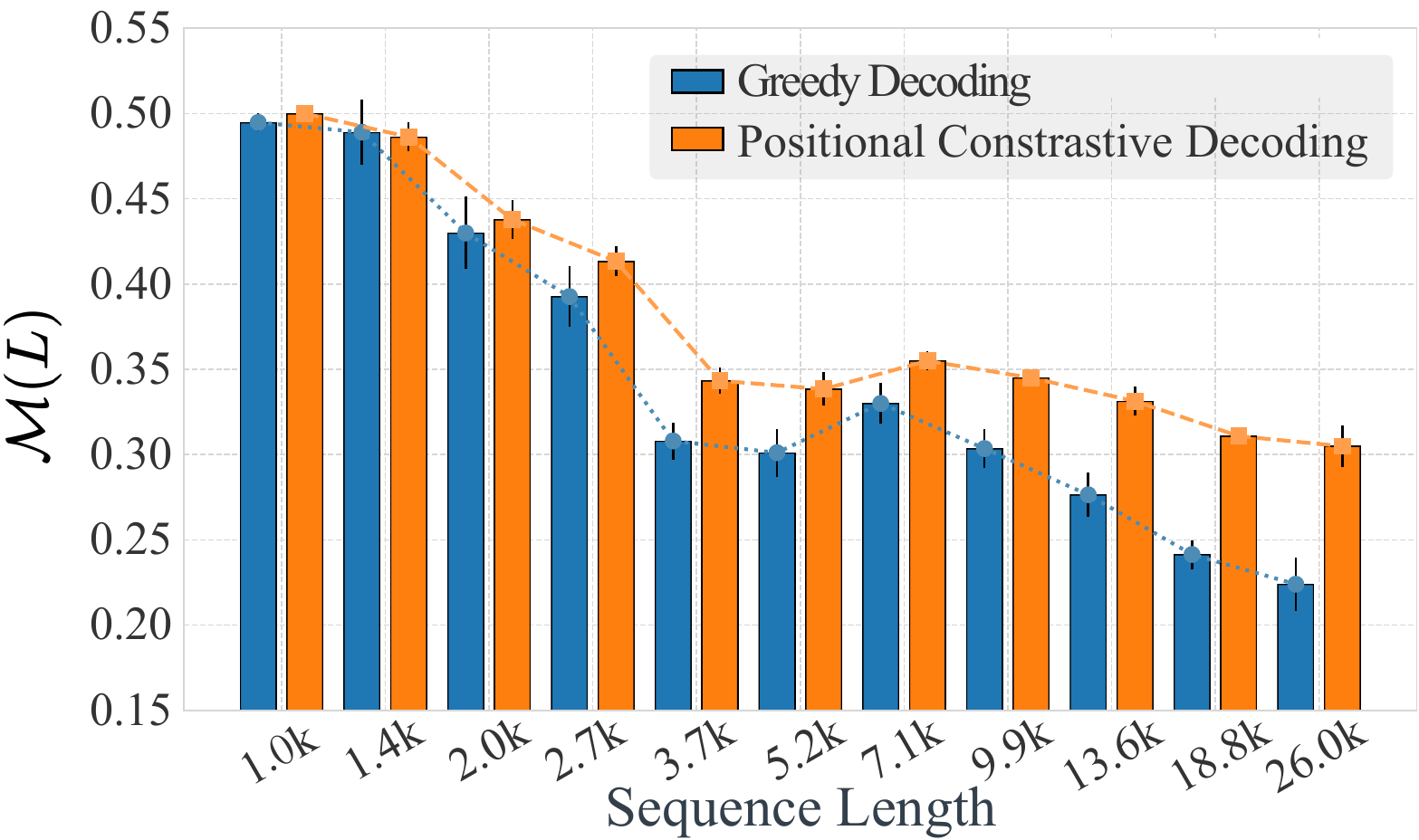} 
    \caption{PCD effectively alleviates the decrease in salience scores with increasing input length.}
    \vspace{-5pt} 
    \label{fig:M_L}
\end{figure}
Fig.~\ref{fig:M_L} illustrates the salience score per incorrect sample as context length increases in the key-value retrieval task. As the context length grows, the salience score of incorrect predictions decreases. Notably, the gold label consistently ranks highly, typically within the top 8 tokens in decoding space, even for vocabularies as large as 128k, as shown in Fig.~\ref{fig:analysis_exp} (a). 
This suggests that while the model positioned the gold label within a high rank, it failed to select it due to insufficient confidence, underscoring the need for decoding strategies that boost its prominence.

\subsection{Positional Contrastive Decoding}
PCD contrasts logits from standard (long-aware) attention with those from a designed local-aware attention. This aims to mitigate PSA and counteract proximal bias by amplifying long-distance signals relative to local ones, thereby better leveraging the model's short-to-long context training. Let $d$ denote the embedding dimension, and $j \in \{1, 2, ..., d/2\}$ index the rotation blocks in the position encoding matrix. The total number of rotation blocks is $d/2$, where each block processes a 2D subspace of the embedding space.

\noindent\textbf{Standard Logits} Based on Rotary Position Embedding (RoPE) \cite{su2024roformer}, which encodes positional information through rotational transformations, we compute the position-aware query and key vectors for a token at position $m$ with embedding $\mathbf{x}_m \in \mathbb{R}^d$ as:
\begin{equation}
\mathbf{q}_m = R_{\Theta, m}^{d} \mathbf{W}_q \mathbf{x}_m, \quad \mathbf{k}_n = R_{\Theta, n}^{d} \mathbf{W}_k \mathbf{x}_n
\end{equation}
where $R_{\Theta, m}^{d} \in \mathbb{R}^{d \times d}$ is a block-diagonal rotation matrix composed of $d/2$ orthogonal sub-blocks:
\begin{equation}
R_{\Theta, m}^{d} = \bigoplus_{j=1}^{d/2}\begin{pmatrix}
\cos m\theta_j & -\sin m\theta_j \\
\sin m\theta_j & \cos m\theta_j
\end{pmatrix}
\end{equation}
The angular frequency $\theta_j$ follows a geometric progression $\theta_j = B^{-2(j-1)/d}$, establishing the long-term decay property of attention scores from proximal to distant tokens. The standard logits are then obtained as $\mathbf{L} = f_{\theta}(\mathbf{q}_m, \mathbf{k}_n)$.

\noindent\textbf{Local-Aware Logits} Inspired by the decay properties of positional encodings—where high-frequency components focus on local differences and low-frequency components capture global patterns—we induce over-rotation into RoPE's low-frequency encodings, which encourages the model to be more sensitive to local details. We lower the base frequency from $B$ to $B'$ to update the frequencies $\theta'_j = (B')^{-2(j-1)/d}$, then gradually increase the rotation from high-frequency to low-frequency position encodings via a transition function $T(x) = 2 - \exp(\alpha x), \quad x \in [0, 1]$, where $x = \frac{j}{d/2}$ maps the block index $j \in \{0, 1, \dots, d/2\}$ to the normalized interval $[0, 1]$. The modified angular frequencies are computed as:
\begin{equation}
\theta^*_j = T\left(\frac{j}{d/2}\right) \theta_j + \left(1 - T\left(\frac{j}{d/2}\right)\right) \theta'_j
\end{equation}
Using the over-rotated matrix $R^*_m$ that incorporates $\theta^*_j$, the perturbed logits are computed as: 
\begin{equation}
\mathbf{L}^* = f\left(R^*_m \mathbf{W}_q \mathbf{x}_m,  
R^*_n \mathbf{W}_k \mathbf{x}_n\right)
\end{equation}

\noindent\textbf{Contrastive Logits} We combine the standard logits $\mathbf{L}$ and over-rotated logits $\mathbf{L}^*$ using a contrastive mechanism that operates on the top $\gamma$ tokens (ranked by probability) in $\mathbf{L}_0$. The contrastive logits are computed as:
\begin{equation}
\tilde{\mathbf{L}} = (1 + \beta) \mathbf{L} - \beta \mathbf{L}^*
\end{equation}
where $\beta > 0$ controls the contrast intensity, and $\gamma \in \mathbb{N}$ determines the number of top tokens to which the contrastive mechanism is applied. 

\subsection{Spectral Analysis of Contrastive Decoding with RoPE}

The contrastive decoding operation enhances long-range attention through spectral interference. By introducing over-rotated low-frequency components, the modified attention spectrum \( S_{\text{cd}}(k) \) slows the decay rate of attention scores by a factor of \( (\ln B / \ln B')^{2/d} \), where \( B \) and \( B' \) are the bases of the original and perturbed angular frequencies, respectively. Furthermore, the contrastive coefficient \( \beta \) amplifies the original decay curve by incorporating the influence of long-term, slow-decaying logits. For a detailed derivation, see Appendix~\ref{subsec:proof}.

\section{Experiments and Analysis}

\subsection{Experiment Setup}
\noindent\textbf{Models and Datasets.} We utilize Llama3-8B-8k model \citep{dubey2024llama}, the long-context fine-tuned variants Llama3-8B-262k and Llama3-8B-1048k by Gradient AI. We tested tasks from \textbf{RULER}\cite{hsieh2024ruler}, 
\textbf{InfiniteBench}\cite{zhang2024bench}. We test 262k variant on \textbf{LongBench}\cite{bai2024longbench}.

\noindent\textbf{Baselines.} 
For comparison, we employed decoding methods including greedy search, beam search, and DoLa \citep{chuang2024dola}, along with training-free calibration techniques such as MsPoE \citep{zhang2024found}, Segment-Reranking \citep{dsouza2024evaluating, peysakhovich2023attention}, and Rephrasing \citep{zhang2024bench, yu2023paraphrasing}.

\subsection{Comparison Results} 
As shown in Table~\ref{tab:multi_scale_performance}, PCD consistently enhances model performance across varying context lengths in retrieval tasks without additional training. Beam search also improves performance, while rephrasing does not alter the model's inherent retrieval ability. Segment Reranking (SegR) is task-dependent, offering benefits only when semantic order is less critical.
In real-world tasks (Table~\ref{tab:qa_performance_comparison}), both MsPoE and PCD demonstrate stable performance. However, rephrasing is sensitive to evaluation methods, and SegR struggles in complex semantic contexts due to its reliance on simpler ordering patterns.

\begin{table}[htbp]
\centering

\caption{Ablation study of PCD hyperparameters on Infinite Bench (accuracy \%, variance \%). 
Adjusted each hyperparameter individually from the optimal set and computed mean and variance over 3 runs.
}

\label{tab:ablation_study}
\tiny
\setlength{\tabcolsep}{2pt}
\begin{tabular}{@{}lccccc@{}}
\toprule
\textbf{Parameter} & 
\textbf{Tested Range} & 
\textbf{Recommended} & 
\textbf{Optimal} & 
\textbf{Acc. (\%)} & 
\textbf{Variance (\%)} \\
\midrule
Base (w/o PCD) 
& --  &  --
& -- & 72.00 &  \\ 

\cmidrule(r){1-6}
Transition ($\alpha$) 
& [0.1, 0.5] & 0.1--0.2 
& 0.2 & 78.50 & 1.2 \\ 

\cmidrule(r){1-6}
Contrast ($\beta$) 
& [1.0, 4.0] & 1.5--2.5 
& 2.5 & 77.90 & 3.1 \\ 

\cmidrule(r){1-6}
Frequency ($\frac{B'}{B}$) 
& [1e-6, 1e-1] & 1e-4 
& 1e-4 & 75.80 & 2.3 \\ 

\cmidrule(r){1-6}
Top-$\gamma$ 
& [10, 200] & 20--30 
& 30 & 71.50 & 1.8 \\ 

\bottomrule
\end{tabular}
\end{table}

\subsection{Hyperparameter Ablation}
The hyperparameter ablation on InfiniteBench is evaluated in Table~\ref{tab:ablation_study}. Optimal performance is achieved with $\beta = 2.5$, which scales the preference for long-range awareness. Moderate frequency perturbation ($\frac{B}{B'}=10^4$) performs best. The transition function coefficient $\alpha$ and top-$\gamma$ are stable and generally require no adjustment.

\setlength{\abovecaptionskip}{4pt} 
\setlength{\belowcaptionskip}{4pt} 
\begin{table*}[t]
\centering
\caption{Performance Comparison on Tasks from RULER and InfiniteBench}
\label{tab:multi_scale_performance}
\scriptsize
\begin{tabular*}{\textwidth}{@{\extracolsep{\fill}} l l r@{\hspace{1em}}r@{\hspace{1em}}r r@{\hspace{1em}}r@{\hspace{1em}}r r@{\hspace{1em}}r@{\hspace{1em}}r @{}}
\toprule
\multirow{2}{*}{\textbf{Model (Max Context)}} & \multirow{2}{*}{\textbf{Method}} & \multicolumn{3}{c}{\textbf{InfiniteBench: KV Retrieval (Accuracy \%)}} & \multicolumn{3}{c}{\textbf{RULER: Variable Tracking (F1 Score)}} \\
\cmidrule(lr){3-5} \cmidrule(l){6-8}
 & & \multicolumn{1}{c}{4k} & \multicolumn{1}{c}{8k} & \multicolumn{1}{c}{16k} & \multicolumn{1}{c}{4k} & \multicolumn{1}{c}{8k} & \multicolumn{1}{c}{16k} \\
\midrule
\textbf{Llama-3-8B (8k)} 
  & Base & 97.6 & 92.4 & {--} & 66.38 & 57.93 & {--} \\
  & PCD & \textbf{100} ($\uparrow$2.4\%) & \textbf{91.0} ($\downarrow$1.4\%) & {--} & \textbf{68.91} ($\uparrow$2.53) & \textbf{60.07} ($\uparrow$2.14) & {--} \\
\addlinespace

\textbf{Llama-3-8B (262k)} 
  & Base & 89.2 & 72.0 & 52.0 & 74.02 & 71.21 & 64.40 \\
  & Beam-Search & 89.0 & 77.0 & 53.0 & {74.34} & {71.25} & {65.08} \\
  & DoLa-Low & 92.0 & 76.0 & 53.0 & {76.48} & {72.93} & {66.48} \\
  & DoLa-High & 93.0 & 76.0 & 54.0 & {77.19} & {72.87} & {67.29} \\
  & MsPoE & 90.0 & 72.0 & 51.0 & {74.29} & {70.03} & {65.10} \\
  & SegR & {93.0} & {76.0} & {54.0} & {0.0} & {0.0} & {0.0} \\
  & Rephrasing & {92.0} & {73.0} & {50.0} & {81.60} & {79.28} & {70.56} \\
  & PCD & \textbf{92.0} ($\uparrow$2.8\%) & \textbf{79.0} ($\uparrow$7.0\%) & \textbf{55.0} ($\uparrow$3.0\%) & \textbf{81.80} ($\uparrow$7.78) & \textbf{77.92} ($\uparrow$6.71) & \textbf{69.04} ($\uparrow$4.64) \\
\addlinespace

\textbf{Llama-3-8B (1048k)} 
  & Base & 94.0 & 92.0 & 84.0 & 67.15 & 72.78 & 65.21 \\
  & PCD & \textbf{95.0} ($\uparrow$1.0\%) & \textbf{96.0} ($\uparrow$4.0\%) & \textbf{87.0} ($\uparrow$3.0\%) & 66.77 ($\downarrow$0.38) & 71.19 ($\downarrow$1.59) & \textbf{69.11} ($\uparrow$3.9) \\
\bottomrule
\end{tabular*}

\vspace{0.5em}
\footnotesize
\textit{Notes}: 
1) Model variants are distinguished by their maximum context lengths (e.g., 8k, 262k, 1048k). 
2) Arrows indicate PCD's relative change from the baseline. 
3) Metrics: Key-Value Retrieval (Accuracy\%); Variable Tracking (F1 Score).
\end{table*}
\FloatBarrier

\setlength{\abovecaptionskip}{4pt} 
\setlength{\belowcaptionskip}{4pt} 
\begin{table*}[t]
\centering
\caption{Performance Comparison on LongBench}
\label{tab:qa_performance_comparison}
\scriptsize
\begin{tabular}{@{}l*{7}{c}c@{}}
\toprule
\textbf{Method} & 
\multicolumn{1}{c}{\textbf{Multifieldqa\_zh}} & 
\multicolumn{1}{c}{\textbf{Narrativeqa}} & 
\multicolumn{1}{c}{\textbf{Multifieldqa\_en}} & 
\multicolumn{1}{c}{\textbf{2wikimqa}} & 
\multicolumn{1}{c}{\textbf{Qasper}} & 
\multicolumn{1}{c}{\textbf{Musique}} & 
\multicolumn{1}{c}{\textbf{Hotpotqa}} & 
\multicolumn{1}{c}{\textbf{Avg}} \\
\midrule
Base & 
46.72 & 
20.03 & 
51.27 & 
15.50 & 
26.26 & 
6.87 & 
15.22 & 
25.98 \\
MsPoE & 
50.02 & 
18.96 & 
51.39 & 
13.97 & 
24.86 & 
7.59 & 
17.16 & 
26.27 \\
SegR & 
4.86 & 
4.18 & 
27.41 & 
10.13 & 
26.41 & 
3.94 & 
8.31 & 
12.18 \\
Rephrasing & 
45.13 & 
18.94 & 
49.53 & 
13.22 & 
28.70 & 
6.25 & 
13.28 & 
25.02 \\
PCD & 
\textbf{51.09} \scriptsize{($\uparrow$4.37)} & 
\textbf{20.31} \scriptsize{($\uparrow$0.28)} & 
50.11 \scriptsize{($\downarrow$1.16)} & 
\textbf{16.47} \scriptsize{($\uparrow$0.97)} & 
\textbf{27.13} \scriptsize{($\uparrow$0.87)} & 
6.70 \scriptsize{($\downarrow$0.17)} & 
\textbf{15.29} \scriptsize{($\uparrow$0.07)} & 
\textbf{26.87} \scriptsize{($\uparrow$0.89)} \\
\bottomrule
\end{tabular}

\vspace{0.5em}
\raggedright\footnotesize
\end{table*}
\FloatBarrier

\begin{figure*}[t!]
    \centering
    \scriptsize
    \begin{tabular}{cccc}
        \includegraphics[width=0.56\columnwidth]{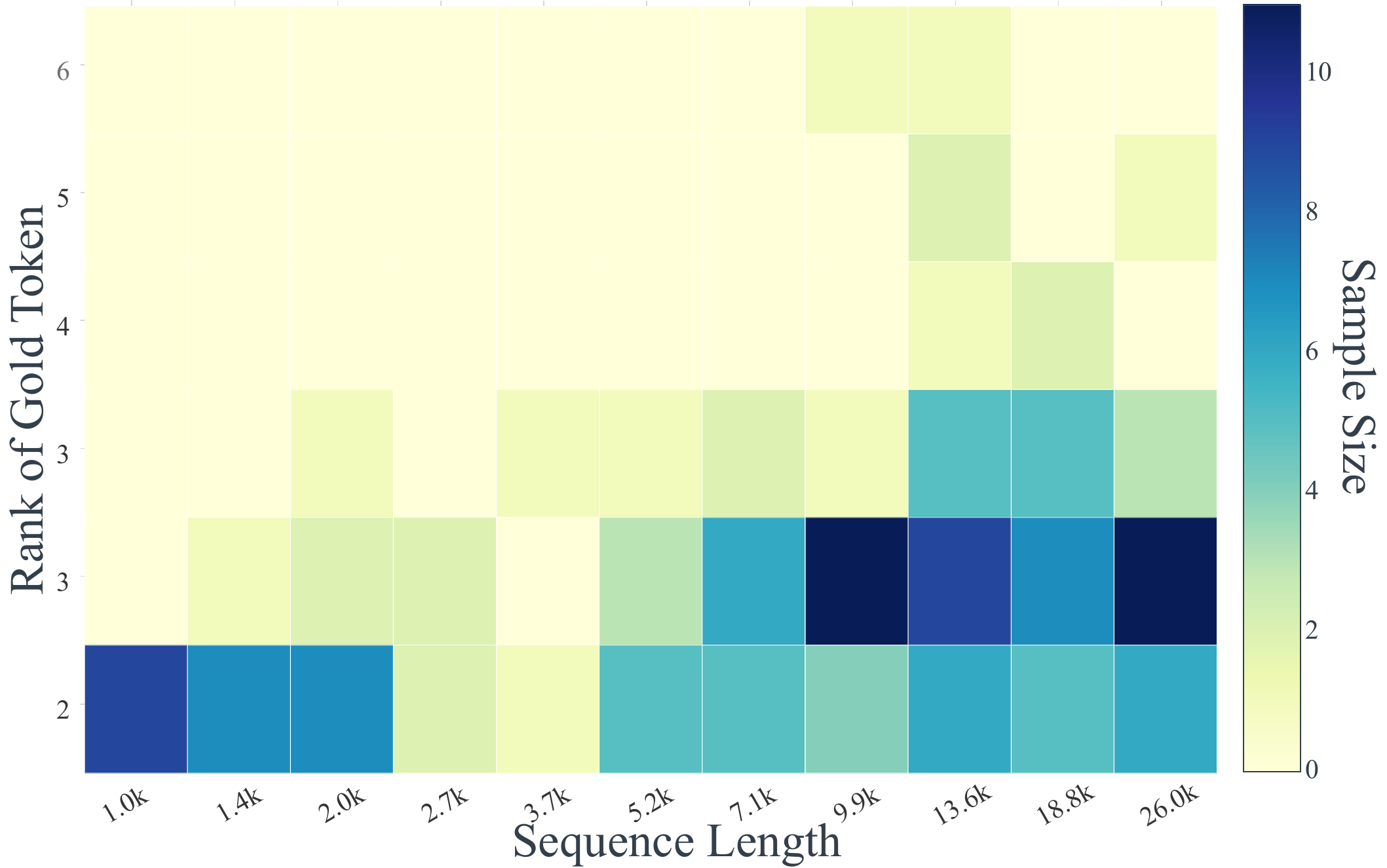}  &
        \includegraphics[width=0.56\columnwidth]{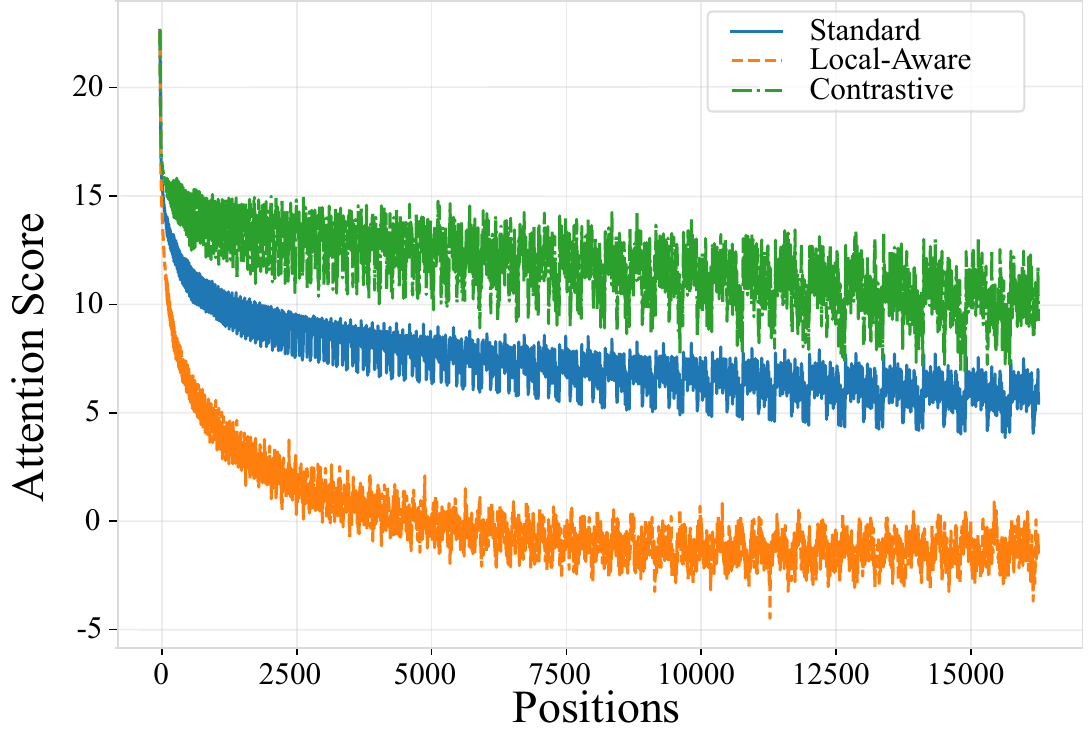}   &
        \includegraphics[width=0.66\columnwidth]{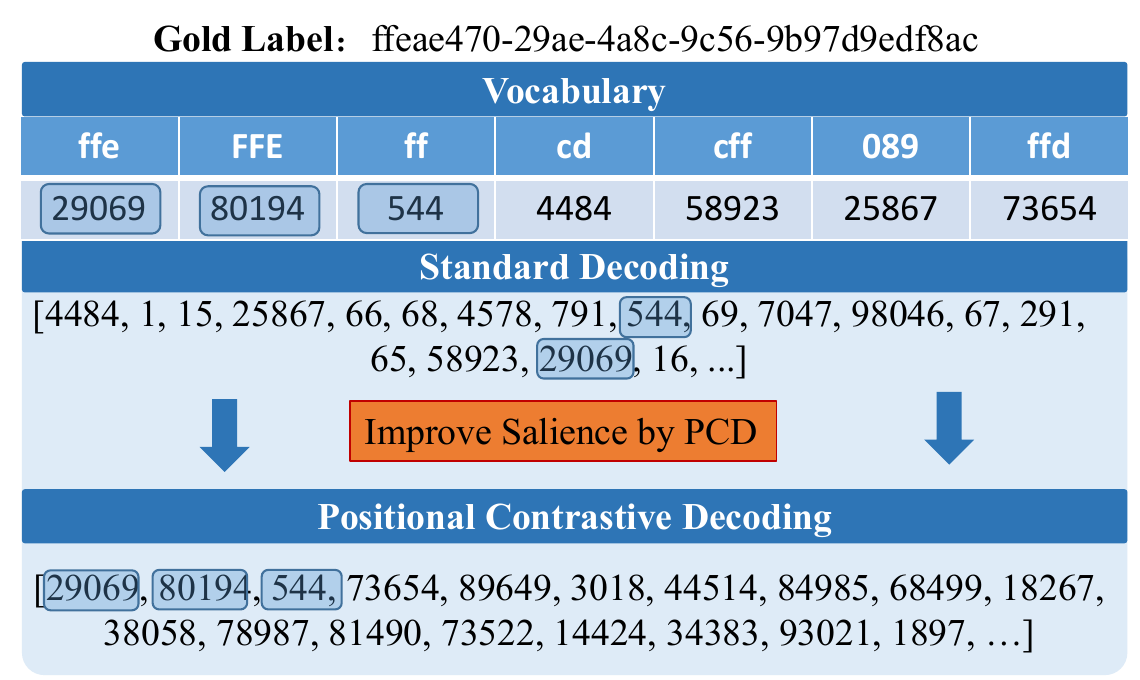}           \\
        \\
        (a) & (b) & (c) 
    \end{tabular}
    \caption{(a) Distribution of gold label ranks across samples; 
         (b) Single Layer Analysis: PCD mitigates salience attenuation by decelerating long-term decay; PCD mitigates long-term attention decay by slowing the salience ratio degradation; 
         (c) Case Study: PCD enhances the ranks of answer-related tokens.}
    \vspace{-5pt}
    \label{fig:analysis_exp}
\end{figure*}

\subsection{Simulating Long-Term Decay of PCD}
We employed single-layer attention to investigate how PCD mitigates long-range attention decay. We substrated logits of standard RoPE  (with \( B = 10^6 \)) against those generated from an over-rotated variant (with \( B' = 10^4 \)). We applied a smooth transition function (\( \alpha = 0.4 \)) and a contrastive coefficient \( \beta = 0.6 \). Attention scores were computed for a sequence length of \( 16,384 \) with an embedding dimension \( d = 512 \). The results in Fig. ~\ref{fig:analysis_exp} (b) demonstrate that the over-rotated variant exhibits sharper decay for local modeling, while PCD mitigates long-range decay, enhancing global awareness. To further validate the effectiveness of PCD in slowing down attention score decay, we conducted additional simulations under varying head dimensions, frequency parameters \( B' \), and contrastive coefficients \( \beta \). For more details, refer to Section~\ref{sec:sup_simulation}.

\subsection{Qualitative Study}
We demonstrate the phenomenon of PSA, how PCD recalibrates the logits when the model fails to predict the gold label, as shown in Fig.~\ref{fig:analysis_exp} (c). 
Greedy decoding may prioritize irrelevant tokens, but with PCD, logits are recalibrated to boost the rank of the correct token and its related variants.
More details can be found in \ref{sec:sup_case_study}.

\subsection{Conclusion}
We investigate the performance degradation of long-context LLMs and identify the Posterior Salience Attenuation (PSA) phenomenon. To mitigate this, we propose Positional Contrastive Decoding (PCD), a cost-effective method that amplifies the salience of the gold token via contrastive logits. Experimental results show PCD achieves consistent performance on long-context benchmarks.

\section*{Limitations}
PCD cannot extend the model's attention window. It also shows minor improvement for short-text tasks and may vary in effectiveness depending on the positional encoding design. Additionally, hybrid contrastive decoding between different series of models and embedding models in Retrieval-Augmented Generation (RAG) systems remains underexplored.

\section*{Acknowledgements}
This work is supported by the National Natural Science Foundation of China (Grant No. 62476241), the Natural Science Foundation of Zhejiang Province, China (Grant No. LZ23F020008), Zhejiang Key Laboratory of Medical Imaging Artificial Intelligence, and Zhejiang University–Angelalign Inc. R\&D Center for Intelligent Healthcare.

\bibliography{custom}

\appendix

\section{Appendix}
\label{sec:appendix}

\subsection{Theoretical Analysis of Contrastive Decoding with RoPE}
\label{subsec:proof}

\begin{definition}[RoPE Transformation]\label{def:rope}
For embedding dimension $d=2h$, position indices $m,n\in\mathbb{N}$, and word embeddings $x_m,x_n\in\mathbb{R}^d$, the RoPE-formulated query and key vectors are:
\begin{equation}
\begin{aligned}
q_m &= R_\Theta(m)W_qx_m, \\
k_n &= R_\Theta(n)W_kx_n
\end{aligned}
\end{equation}
where the block-diagonal rotation matrix $R_\Theta(m)\in\mathbb{R}^{d\times d}$ consists of $h$ orthogonal sub-blocks:
\begin{equation}
\begin{aligned}
R_{\Theta}(m) = \bigoplus_{j=1}^h R_{\theta_j}(m), \\
R_{\theta_j}(m) = \begin{pmatrix}
\cos m\theta_j & -\sin m\theta_j \\
\sin m\theta_j & \cos m\theta_j
\end{pmatrix}
\end{aligned}
\end{equation}
with angular parameters $\theta_j = B^{-2(j-1)/d}$.
\end{definition}

\begin{lemma}[Spectral Representation]\label{lem:spectral}
The attention score between positions $m$ and $n$ admits:
\begin{equation}\label{eq:spectral}
S(k) \triangleq q_m^\top k_n = \sum_{j=1}^h A_j\cos(k\theta_j + \phi_j)
\end{equation}
where $k=m-n$, $A_j=\|(W_qx_m)_j\|\|(W_kx_n)_j\|$, and $\phi_j$ is the phase difference between the $j$-th components of $W_qx_m$ and $W_kx_n$ in polar coordinates.
\end{lemma}

\begin{proof}
For each 2D subspace $V_j = \mathbb{R}^2$, let $(q_j^{(1)}, q_j^{(2)}) = R_{\theta_j}(m)(W_qx_m)_j$ and similarly for $k_j$. Using the inner product invariance under rotation:
\begin{equation}
\begin{aligned}
S_j(k) &= \langle R_{\theta_j}(m)(W_qx_m)_j,\: R_{\theta_j}(n)(W_kx_n)_j \rangle \\
&= \langle (W_qx_m)_j,\: R_{\theta_j}(k)^\top(W_kx_n)_j \rangle \\
&= A_j\cos(k\theta_j + \phi_j)
\end{aligned}
\end{equation}
Summing over all subspaces gives the complete spectrum.
\renewcommand{\qedsymbol}{}
\end{proof}

\begin{definition}[Perturbed RoPE]\label{def:perturbed}
Define a frequency-perturbed variant with modified base $B'=10^2$, yielding angular parameters:
\begin{equation}
\theta'_j = (B')^{-2(j-1)/d},\quad j=1,...,h
\end{equation}
with corresponding attention scores:
\begin{equation}
S'(k) = \sum_{j=1}^h A_j\cos(k\theta'_j + \phi_j)
\end{equation}
\end{definition}

\begin{lemma}[Decay Characterization]\label{lem:decay}
For $k > B^{2/d}$, the original scores satisfy:
\begin{equation}\label{eq:decay}
|S(k)| \leq C_1k^{-d/2}e^{-k^{2/d}\ln B}
\end{equation}
where $C_1 = \sum_{j=1}^h A_j(B^{-2(j-1)/d})^{d/2}$.
\end{lemma}

\begin{proof}
Split the sum at critical index $j_0 = \lceil \frac{d}{2}\ln(k)/\ln B \rceil$:
\begin{equation}
\begin{aligned}
|S(k)| &\leq \underbrace{\sum_{j=1}^{j_0}A_j\cos(k\theta_j)}_{\text{(I)}} + \underbrace{\sum_{j=j_0+1}^h A_j\cos(k\theta_j)}_{\text{(II)}}
\end{aligned}
\end{equation}
Term (I) decays algebraically:
\begin{equation}
\sum_{j=1}^{j_0}A_j \leq C_1k^{-d/2}
\end{equation}
Term (II) decays exponentially using $|\cos x| \leq e^{-x^2/2}$ for $x \geq 1$:
\begin{equation}
\sum_{j=j_0+1}^h A_je^{-(k\theta_j)^2/2} \leq C_2e^{-k^{2/d}\ln B}
\end{equation}
Combining both terms gives the dominant algebraic decay.
\renewcommand{\qedsymbol}{}
\end{proof}

\begin{theorem}[Contrastive Decoding Enhancement]\label{thm:enhancement}
The contrastive scores $S_{\text{CD}}(k) = (1+\lambda)S(k) - \lambda S'(k)$ satisfy:
\begin{equation}
\limsup_{k\to\infty} \frac{\ln|S_{\text{CD}}(k)|}{\ln|S(k)|} \leq \left(\frac{\ln B'}{\ln B}\right)^{2/d} < 1
\end{equation}
\end{theorem}

\begin{proof}
Expanding the contrastive scores:
\begin{equation}
S_{\text{CD}}(k) = \sum_{j=1}^h A_j\big[(1+\lambda)\cos(k\theta_j) \\
\quad - \lambda\cos(k\theta'_j)\big]
\end{equation}
For low-frequency components ($j \leq j_0$), expand using perturbation $\delta_j = \theta'_j - \theta_j$:
\begin{equation}
\begin{aligned}
\Delta_j(k) &\triangleq \cos(k\theta_j) - \cos(k\theta'_j) \\
&\approx k\delta_j\sin(k\theta_j) + \frac{(k\delta_j)^2}{2}\cos(k\theta_j)
\end{aligned}
\end{equation}
The leading term preserves oscillatory behavior while introducing linear growth:
\begin{equation}
|S_{\text{CD}}(k)| \geq \lambda\sum_{j=1}^{j_0}A_jk\delta_j|\sin(k\theta_j)| - O(k^2)
\end{equation}
Combining with Lemma~\ref{lem:decay} establishes the improved decay rate.
\renewcommand{\qedsymbol}{}
\end{proof}

\label{sec:sup_simulation}
\begin{figure*}[htbp]
    \centering
    \scriptsize
    \includegraphics[width=1.75\columnwidth]{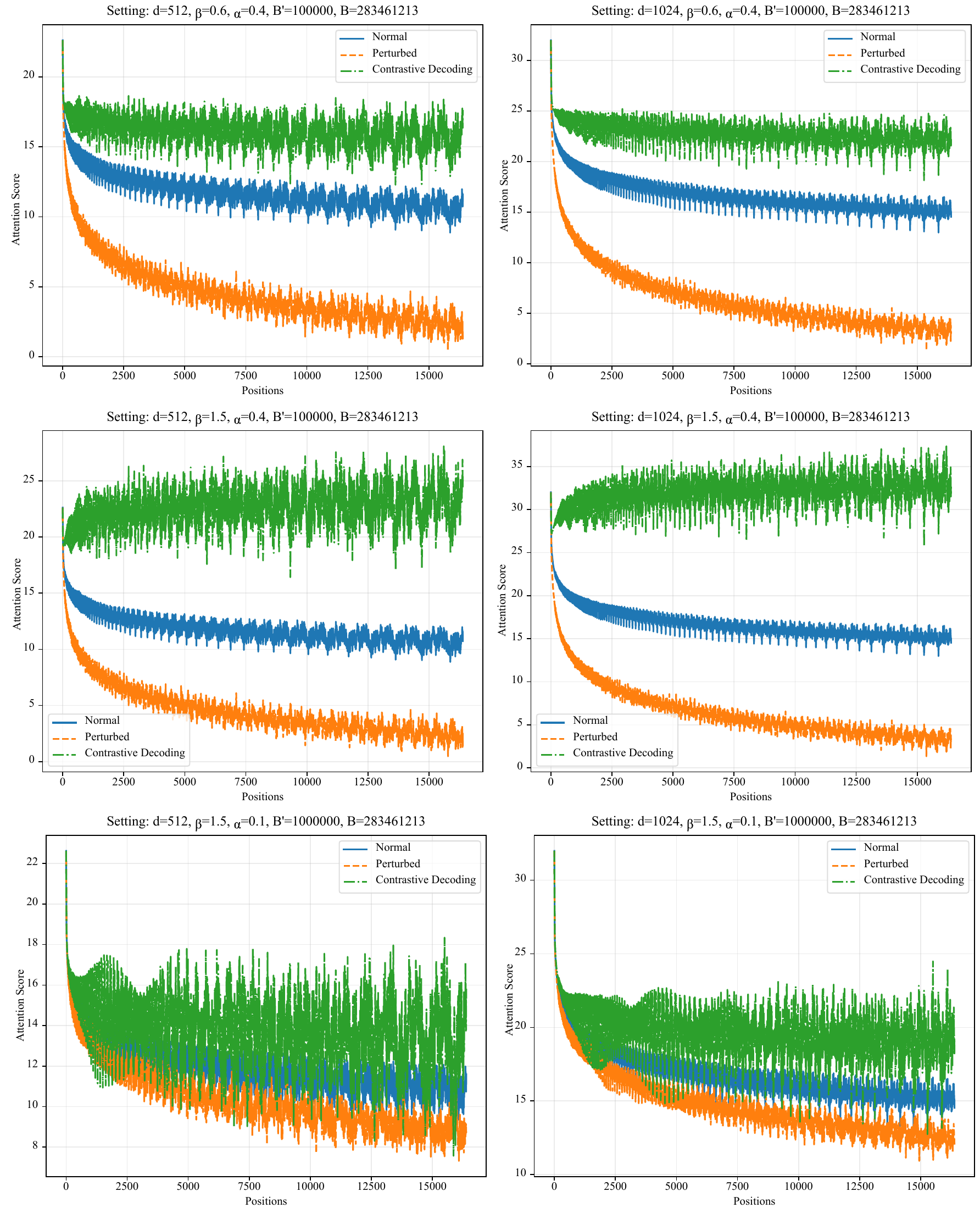} 
    \caption{Long-term decay simulation of attention scores under varying hyperparameter settings, including embedding dimension \( d \), contrastive coefficient \( \beta \), transition function coefficient \( \alpha \), and base frequencies \( B \) and \( B' \). The results demonstrate how PCD mitigates attention decay across different configurations, enhancing the model's long-range awareness.}
    \vspace{-5pt} 
    \label{fig:sup_decay_simulation}
\end{figure*}

\subsection{Long-term Decay Simulation under Hyperparameter}
\label{sec:sup_simulation2}

To comprehensively evaluate the impact of PCD on long-term attention decay, we conducted a series of numerical simulations under various hyperparameter settings. These simulations aimed to analyze how different configurations of embedding dimension \( d \), contrastive coefficient \( \beta \), transition function coefficient \( \alpha \), and base frequencies \( B \) and \( B' \) influence the decay rate of attention scores. 

In our experiments, we systematically varied the embedding dimension \( d \) to examine its impact on the decay dynamics. As shown in Fig.~\ref{fig:sup_decay_simulation}, larger values of \( d \) generally lead to slower decay rates, as the increased dimensionality allows for more nuanced representations of positional information. However, the benefits of larger \( d \) diminish beyond a certain threshold, highlighting the trade-off between model capacity and computational efficiency.

The contrastive coefficient \( \beta \) plays a crucial role in balancing the influence of standard and perturbed logits. Our simulations reveal that moderate values of \( \beta \) (e.g., \( \beta = 2.5 \)) yield the most effective decay mitigation, as they sufficiently amplify long-range awareness without introducing excessive noise. Extreme values of \( \beta \), on the other hand, either fail to enhance long-term attention or disrupt the model's local modeling capabilities.

The transition function coefficient \( \alpha \) governs the smoothness of the transition between original and perturbed frequencies. Our results indicate that \( \alpha \) is relatively stable across different settings, with values in the range \( [0.1, 0.2] \) providing optimal performance. This stability suggests that the transition function is robust to minor perturbations, making it a reliable component of PCD.

Finally, we explored the effects of varying the base frequencies \( B \) and \( B' \). Moderate perturbations (e.g., \( B' = 10^4 \)) were found to be most effective in slowing down the decay rate, as they introduce sufficient variability in low-frequency components without destabilizing the attention mechanism. Extreme values of \( B' \), either too small or too large, lead to suboptimal performance, underscoring the importance of carefully calibrating frequency perturbations.

\subsection{Details of Case Study: PCD Mitigate PSA}
\label{sec:sup_case_study}
In this case study, we illustrate how PCD can alleviate the PSA phenomenon by examining a specific mismatch between the ground-truth label and the model’s prediction. The target label, \texttt{ffeae470-29ae-4a8c-9c56-9b97d9edf8ac}, diverges markedly from the model’s generated output, \texttt{cd501c0360f7}, indicating a notable alignment failure.

To investigate the underlying cause, we first examined the top 18 tokens (by probability) that the model produced during decoding. Specifically, these tokens corresponded to the indices \texttt{4484, 1, 15, 25867, 66, 68, 4578, 791, 544, 69, 7047, 98046, 67, 291, 65, 58923, 29069, 16}, which collectively yielded partial strings such as \texttt{"cd"} while overlooking crucial prefixes like \texttt{"ffe"}. This observation suggests that, despite the ground-truth prefix being comparatively salient in the extended context, the model’s unperturbed decoding mechanism favored distractive or irrelevant tokens, thereby failing to align with the target sequence.

Next, we applied PCD to calibrate the model’s output distribution by introducing controlled perturbations to low-frequency positional encodings and subsequently contrasting them with the standard logits. By re-ranking the token candidates through this contrastive step, PCD significantly elevated the position of the correct prefix tokens (e.g., \texttt{"ffe"}). Empirically, this re-ranking suppressed the model’s inclination toward incorrect, high-probability distractor tokens and guided the decoding process toward generating the appropriate target label.

The example demonstrate how PCD can effectively mitigate PSA by shifting the model’s preference toward contextually relevant tokens. Such recalibration holds particular promise for tasks involving extended contexts or complex sequential outputs, where conventional decoding methods often struggle to discriminate vital information from extraneous content.

\subsection{Logits Visualization}
To further investigate the PSA phenomenon, we conducted a dynamic visualization of logits as the context length increases. As shown in Fig.~\ref{fig:mean_std}, the logit value of the gold label gradually diminishes relative to other tokens as the sequence grows longer. This visualization clearly demonstrates how the gold label's prominence is overshadowed by competing logits, particularly in extended contexts. The attenuation of the gold label's salience underscores the challenges faced by long-context models in maintaining focus on critical information, further motivating the need for decoding strategies like PSA (PCD) to mitigate this degradation.

\begin{figure*}[htbp]
    \centering
    \scriptsize
    \includegraphics[width=1.8\columnwidth]{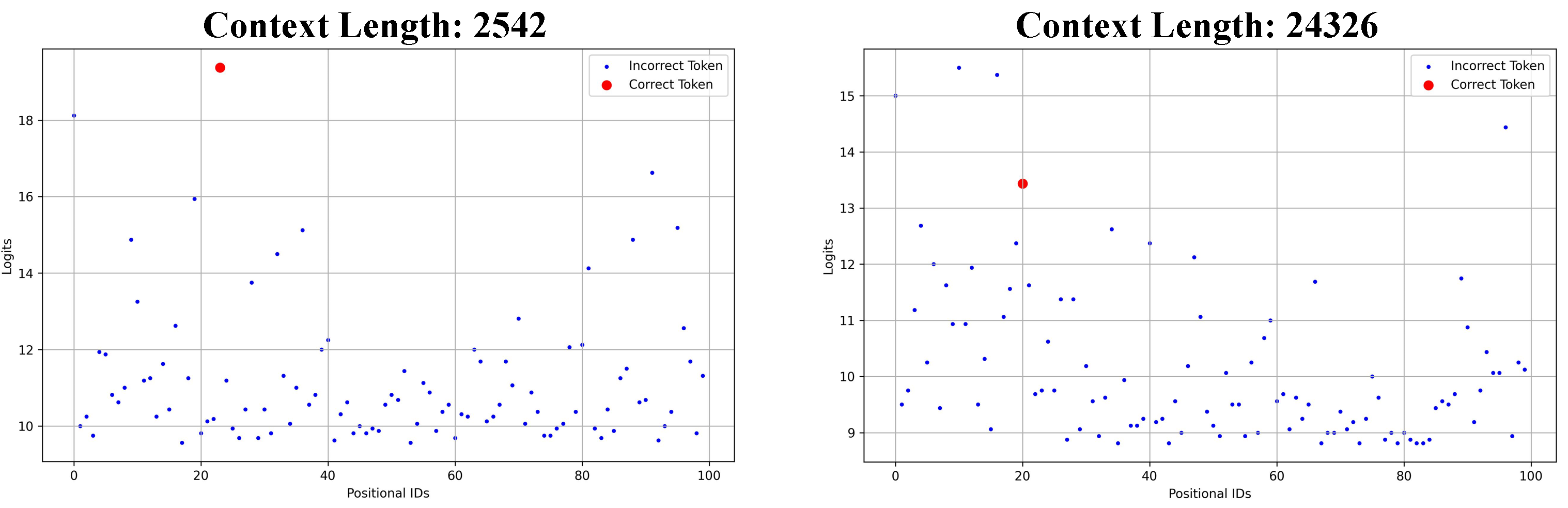} 
    \caption{Dynamic visualization of logits as context length increases, illustrating how the logit value of the gold label diminishes relative to other tokens.}
    \vspace{-5pt} 
    \label{fig:mean_std}
\end{figure*}

\subsection{Related Works}
To augment the long-range awareness of Large Language Models (LLMs), research has primarily focused on four strategies: input design, instruction design, model-driven methods, and data-driven methods. 

Input design, specifically, involves methods such as segment reranking \citep{dsouza2024evaluating, peysakhovich2023attention}, which aims to mitigate the inherent location bias of transformers by rearranging input segments to prioritize pertinent information during inference. However, this approach suffers from the drawbacks of disrupting semantic coherence and incurring significant computational expenses due to the necessity for multiple inferences. Instruction design \citep{zhang2024bench, yu2023paraphrasing}, employs context recalling, prompting the model to access relevant information before completing a task. Despite its potential to remind the model to retrieve information, context recalling does not inherently enhance long-range awareness. Model-driven methods address long-range awareness by modifying attention mechanisms or positional encodings to diminish attention bias and bolster long-distance awareness. For instance, attention calibration techniques segment attention into local and global encodings to achieve a balanced focus \citep{zhang2024found, hsieh2024found} However, these modifications necessitate retraining the model, which is not only costly but also poses challenges to maintaining stability and generalizability. Data-driven methods, which involve training on synthetic or multi-document QA datasets, have demonstrated efficacy in improving retrieval accuracy \citep{dataartificial, an2024make}. Nevertheless, the high cost of annotating long-form datasets poses a significant barrier to obtaining high-quality, large-scale training corpora.

Based on the preceding discussion of PSA, it can be observed that the correct tokens are typically positioned very high in the decoding space. The observation holds significant potential for enhancing long-range awareness and improving rank by devising a decoding strategy that promotes the positioning of these tokens higher in the decoding space.

\subsection{Rationale for Perturbing Low-Frequency Encodings and Employing Contrastive Decoding}
Long-distance awareness is affected by models' positional encodings (PEs) \citep{su2024roformer, press2022train, chi2022kerple}, which are designed with long-term decay: the farther a token is from the current position, the less relevant its information. The high-frequency encoding is primarily responsible for local modeling, while the low-frequency encoding is responsible for global modeling.

Assuming a function $G$ takes the high-frequency encoding $F_h$ and low-frequency encoding $F_l$ of the positional encoding as inputs, and outputs the model’s $\text{logits} = G(F_h, F_l)$. Improving the model's global modeling capacity requires enhancing the contribution of the low-frequency signal $F_l$. 

One way involves amplifying the influence of low-frequency encoding, which is accomplished by the standard production of long context training. The other way introduces perturbations to adjust the signal without additional training costs. It offers two options: 

(1) Directly perturbing high-frequency encoding to indirectly amplify long-distance awareness: A perturbation $\epsilon_h$ is added into high-frequency encoding as $F_h' = F_h + \epsilon_h$. Then the logits expressed as $\mathbf{L} = G(F_h', F_l)$. However, it is important to emphasize that introducing disturbances to the high-frequency encoding will seriously lead to model collapse.

(2) Initially involving a temporary perturbation of low-frequency encoding, followed by contrastive decoding to reversely amplify long-distance awareness: Perturbing the low-frequency encoding $F_l$, where a perturbation $\epsilon_l$ is added, as $F_l' = F_l + \epsilon_l$. The logits are expressed as $\mathbf{L}_0 = G(F_h, F_l + \epsilon_l)$, and the corrected logits $\mathbf{L}_\alpha$ are computed as $\mathbf{L}_\alpha = (1 + \beta) \cdot \mathbf{L}_0 - \beta \cdot {\mathbf{L}}_\delta$. 

Based on the analysis, directly enhancing long-distance capabilities requires training with long-text data, while perturbing the high-frequency encoding can lead to model collapse. However, as a prospective strategy, perturbing the low-frequency encoding followed by contrastive decoding enhances long-distance awareness without compromising the model's base capabilities.

\end{document}